\documentclass[letterpaper, 10 pt, conference]{ieeeconf}  % Comment this line out if you need a4paper

\IEEEoverridecommandlockouts                              % This command is only needed if 
\usepackage{graphics} % for pdf, bitmapped graphics files
\usepackage{xcolor}
\usepackage{graphicx} % for pdf, bitmapped graphics files
\usepackage{mathptmx} % assumes new font selection scheme installed
\usepackage{times} % assumes new font selection scheme installed
\usepackage{amsmath} % assumes amsmath package installed
\usepackage{amssymb}  % assumes amsmath package installed
\usepackage{comment}
\usepackage{url}

\usepackage{amsthm}

\newtheorem{prbm}{Problem}
\newtheorem{thm}{\protect\theoremname}

\theoremstyle{definition}
\newtheorem{defn}{\protect\definitionname}

\theoremstyle{remark}
\newtheorem{claim}{\protect\claimname}
\newtheorem{rem}{\protect\remarkname}

\usepackage{babel}
\providecommand{\claimname}{Claim}
\providecommand{\definitionname}{Definition}
\providecommand{\remarkname}{Remark}
\providecommand{\theoremname}{Theorem}

\title{\LARGE \bf
Spiking Neuromorphic Control for Periodic Orbits Stabilization in Underactuated Robots: A Greedy Lyapunov-Based Approach
}

\author{Surov M.$^{1}$ and Sumenkov O.$^{1}$% <-this % stops a space
\thanks{$^{1}$The authors are with Mathematical Robotics Science Division, Sirius University of Science and Technology, Olympiyskiy Av. 1, Sirius, 354340 Russia
        {\tt\small \{surov.mo, sumenkov.oy\}@talantiuspeh.ru}.}%
}

\begin{document}

\maketitle
\thispagestyle{empty}
\pagestyle{empty}

%%%%%%%%%%%%%%%%%%%%%%%%%%%%%%%%%%%%%%%%%%%%%%%%%%%%%%%%%%%%%%%%%%%%%%%%%%%%%%%%
\begin{abstract}
    This work investigates the use of event-based spiking controllers, implemented as spiking neural networks, for trajectory tracking in underactuated robots. Specifically, we address the problem of stabilizing a prescribed periodic orbit. It is well known that, by applying the method of transverse linearization to the robot dynamics in a neighborhood of a periodic trajectory, the dynamics can be approximately described by a linear time-varying system with periodic coefficients (a periodic LTV system). The problem of orbital stabilization can therefore be reduced to stabilizing the trivial solution of this LTV system. We show that a previously proposed Lyapunov-based greedy design, originally developed for LTI systems, can also be extended to the corresponding linearized transverse dynamics, which take the form of a periodic LTV system. This enables the construction of a feedback controller for orbital stabilization of the robot. The effectiveness of the proposed method is demonstrated on the trajectory stabilization problem for the butterfly robot.
\end{abstract}

%%%%%%%%%%%%%%%%%%%%%%%%%%%%%%%%%%%%%%%%%%%%%%%%%%%%%%%%%%%%%%%%%%%%%%%%%%%%%%%%
\section{INTRODUCTION}
    
    Neuromorphic control is an emerging paradigm in control systems engineering, inspired by the event-driven information processing of the human brain. Although the underlying principles of spike-based computation and neuromorphic engineering date back to the 1990s~\cite{astrom_spike,mead1990}, their application to control systems has gained renewed attention in recent years, with successful experimental demonstrations reported, for example, in~\cite{gallego2020event,krauhausen2021}. Thus, spike-based control shows potential for energy-efficient, robust, and computationally simple control, motivating further investigation of its use in control systems.

    Despite this progress, a significant gap remains between practical implementations and the theoretical justification of neuromorphic control systems. Several studies in neuromorphic control~\cite{goebel2009,singh2018regulation,sepulchre2022,petri2025cdc,medvedeva2025cdc,agliati2026spiking,petri2026nahs,heemels2026ejc}, among others, have addressed this gap by developing stability analysis tools for event-driven systems.

    In particular, for LTI systems, several approaches have recently been proposed. The \emph{emulation-based} approach~\cite{petri2025cdc,petri2025emulation,petri2026nahs} uses integrate-and-fire neurons with fixed-amplitude spikes and provides practical stability guarantees for the resulting closed-loop system. An alternative, referred to as the \emph{greedy Lyapunov-based} approach~\cite{agliati2026spiking,heemels2026ejc}, generates neuron spikes according to a local minimization principle, enabling Lyapunov-based stability analysis of the resulting hybrid closed-loop dynamics.

    Building on recent results for LTI systems~\cite{heemels2026ejc}, we extend spiking neuromorphic control, in which the control input is represented as a sum of constant-amplitude spikes, to the stabilization of periodic orbits in underactuated robots. This control problem arises in many robotic and engineering applications, including walking bipeds~\cite{grizzle2001}, pendulum-like systems~\cite{jahn2021}, spacecraft~\cite{spacecraft2026}, and others. Several approaches have been developed to address this problem, including transverse linearization~\cite{banaszuk1995}, energy shaping~\cite{ortega_energy_shaping}, eigenmanifold stabilization~\cite{della2020}, and many others. The underactuated nature of these systems, with fewer control signals than degrees of freedom, makes the stabilization problem challenging. Moreover, when the control signals are restricted to sums of constant-amplitude spikes, it is not immediately clear whether the tracking problem admits a solution or what stabilization properties can be achieved with this type of controller in principle.

    This paper demonstrates that a \emph{spiking neuromorphic controller} emitting fixed-amplitude pulses can achieve \emph{practical stabilization} of a reference periodic orbit. We show that the dynamics of a robot with spiking inputs can be approximated by a hybrid system expressed in local transverse coordinates near the target orbit, reducing the original problem to the stabilization of an auxiliary hybrid LTV system. To stabilize this auxiliary system, we employ the greedy Lyapunov-based approach originally developed for LTI systems~\cite{heemels2026ejc}. Building on this framework, we provide a constructive design procedure for spiking neuromorphic controllers and establish practical stability of the reference orbit. The approach is experimentally validated on a butterfly robot testbed.

    The paper is organized as follows. Section II presents the problem statement and preliminaries on orbital stability and transverse linearization. Section III presents the main result on greedy Lyapunov-based spiking neuromorphic control design. Section IV presents experimental results on the stabilization of a periodic motion of a butterfly robot using a spiking neuromorphic controller. Finally, Section V concludes the paper.

\section{Problem Statement}
\label{sec:problem_stm}

    We consider an underactuated robot~\cite{spong}
    whose dynamics are described in control-affine form as
    \begin{equation}
        \dot{x}=f\!\left(x\right)+g\!\left(x\right)u,\label{eq:sys-dynamics}
    \end{equation}
    where the even-dimensional phase variable $x\equiv\left(q,\dot{q}\right)\in\mathbb{R}^{n}$
    is composed of the generalized coordinates $q\in\mathbb{R}^{\frac{n}{2}}$
    and generalized velocities $\dot{q}\equiv\frac{dq}{dt}$; $u\in\mathbb{R}$
    is the control input. The functions $f,\,g:\mathbb{R}^{n}\to\mathbb{R}^{n}$
    are continuously differentiable.
    When the robot dynamics are described by the Euler--Lagrange equations,
    the functions $f$ and $g$ take the form
    {\setlength{\arraycolsep}{2pt}\begin{align*}
        f\!\left(q,\dot{q}\right)  :=\left(\begin{array}{c}
        \dot{q}\\
        M^{-1}\!\left(q\right)N\left(q,\dot{q}\right)
        \end{array}\right), \;\;
        g\!\left(q\right)  :=\left(\begin{array}{c}
        0\\
        M^{-1}\!\left(q\right)F\!\left(q\right)
        \end{array}\right),
    \end{align*}}where $M\!\left(q\right)\in\mathbb{R}^{\frac{n}{2}\times\frac{n}{2}}$
    is a non-degenerate kinetic energy matrix, $N\!\left(q,\dot{q}\right)\in\mathbb{R}^{\frac{n}{2}}$
    describes Coriolis, centrifugal, and gravitational forces, and $F\left(q\right)\in\mathbb{R}^{\frac{n}{2}}$
    maps the control signal $u$ to the generalized force.
    
    Let $x=x_{\star}\left(t\right),$ $u=u_{\star}\left(t\right)$ be
    $T$-periodic functions satisfying~(\ref{eq:sys-dynamics}). The set
    of phase-space points through which the trajectory $x_{\star}$ passes is
    referred to as the \emph{orbit}~\cite{khalil}
    \[
        \mathrm{orb}\,x_{\star}:=\left\{ x_{\star}\left(t\right)\mid t\in\mathbb{R}\right\} \subset \mathbb{R}^n.
    \]
    Imposing the standard Euclidean structure on the phase space, we also
    introduce the distance operator between two points $a,b\in\mathbb{R}^{n}$
    as 
    \[
        \mathrm{dist}\left(a,b\right) := \left\Vert a - b \right\Vert,
    \]
    and the distance between a point $x$ and the orbit $\mathrm{orb}\,x_{\star}$
    as 
    \[
        \mathrm{dist}\left(x,\,\mathrm{orb}\,x_{\star}\right):=\min_{\nu\in\left[0,T\right]}\mathrm{dist}\left(x,x_{\star}\left(\nu\right)\right).
    \]
    We assume that the trajectory $x_{\star}$ does not have self-intersections; then the orbit is a topological set of dimension 1, and there exists a tubular neighborhood $U_{\varepsilon}(x_\star)$, which is defined as
    \[
        U_{\varepsilon}\!\left(x_{\star}\right):=\left\{ x\in\mathbb{R}^{n}\mid\mathrm{dist}\left(x,\,\mathrm{orb}\,x_{\star}\right)<\varepsilon\right\}.
    \]
    The standard definition of the orbital stability~\cite{khalil}
    means that for any arbitrarily small tube $U_{\varepsilon}(x_\star)$
    there exists a tube $U_{\delta}(x_\star)$, $\delta\leq\varepsilon$, such
    that any solution of the dynamical system originating in $U_{\delta}(x_\star)$
    does not leave $U_{\varepsilon}(x_\star)$. Asymptotic orbital stability additionally
    requires attractivity: any solution $x\left(t\right)$ converges asymptotically to the orbit, i.e., $\mathrm{dist}\left(x\left(t\right),\,\mathrm{orb}\,x_{\star}\right)\to0$
    as $t\to\infty.$
    
    In this work, we address practical orbital stability, which involves boundedness of solutions near the orbit and attractivity of a neighborhood $U_{\rho}\!\left(x_{\star}\right)$. While practical (asymptotic) stability of an equilibrium is a well-established notion in nonlinear control (see, e.g.,~\cite{Teel-1999,Chaillet-2008}), there is, to the best of our knowledge, no universally adopted definition of practical orbital stability. We therefore introduce the following definition.
    \begin{defn}
    \label{def:practical-orbital-stability} 
        We say that the trajectory
        $x_{\star}\left(t\right)$ of the equation
        \begin{equation}
            \dot{x}=\phi\left(x\right)\label{eq:ode-1}
        \end{equation}
        is \emph{practically (asymptotically) orbitally stable} if there exists $\rho > 0$, referred
        to as the \emph{practical accuracy}, such that the following conditions
        hold:
        \begin{enumerate}
            \item 
                For any $\varepsilon \geq \rho$, there exists $\delta = \delta\left(\varepsilon\right)>0$
                such that any solution of~(\ref{eq:ode-1}) with initial condition
                $x\left(t_0\right)\in U_{\delta}\!\left(x_{\star}\right)$ satisfies $x\left(t\right)\in U_{\varepsilon}\!\left(x_{\star}\right)$
                for all $t\geq t_0$;
            \item 
                Any solution of~(\ref{eq:ode-1}) eventually reaches the tube $U_{\rho}\!\left(x_{\star}\right)$
                and stays inside it, i.e., there exists $t_{\rho}\geq t_{0}$ such that
                $\mathrm{dist}\left(x\left(t\right),\,\mathrm{orb}\,x_{\star}\right)\leq \rho$ for all $t\geq t_{\rho}$.
        \end{enumerate}
    \end{defn}
    We consider practical orbital stability in a local sense: for initial conditions far from the orbit, the requirements may fail. Accordingly, we assume that the conditions of Definition~\ref{def:practical-orbital-stability} hold only for initial conditions in some subset of $\mathbb{R}^n$ containing the orbit.
    
    Based on this definition, we formulate the problem addressed in this paper. 
    \begin{prbm}
    \label{prob:feedback-design}
        Design a state feedback $u=u\left(q,\dot{q}\right)$ for the system~(\ref{eq:sys-dynamics}) that ensures practical orbital stability of the desired trajectory $x_{\star}(t)$.
    \end{prbm}

    For completeness of the exposition, we briefly recall the main steps of the transverse linearization approach, which will later serve as the basis for constructing an event-triggered impulsive control feedback for solving Problem~\ref{prob:feedback-design}.

\subsection{Robot Dynamics Linearization Near a Reference Orbit}
\label{sec:transverse-coordinates-def}

    Let 
    \begin{equation}
        x=\alpha\left(\tau,\xi\right)\label{eq:alpha}
    \end{equation}
    be a coordinate transformation, where $\tau\in\mathbb{R}$ is the
    \emph{projection variable} and $\xi\in\mathbb{R}^{n-1}$ are the \emph{transverse
    variables}. We assume the following:
    \begin{itemize}
        \item 
            The function $\alpha\in C^{2}$ is periodic in $\tau$: $\alpha\left(\tau+T,\xi\right)=\alpha\left(\tau,\xi\right),\forall\tau\in\mathbb{R},\,\xi\in\mathbb{R}^{n-1}.$
        \item 
            For all $\tau\in\mathbb{R}$ the map $\alpha$ satisfies $x_{\star}\left(\tau\right)=\alpha\left(\tau,0\right)$.
        \item 
            There exists $\varepsilon>0$ such that $\alpha$ induces a local
            diffeomorphism from $\left(\mathbb{R}/T\mathbb{Z}\right)\times\mathbf{B}_{\varepsilon}\left(0\right)$
            onto $U\left(x_{\star}\right)$. Here, $\mathbb{R}/T\mathbb{Z}$ is
            the quotient space defined as $\mathbb{R}$ equipped with the equivalence
            relation $\sim$: two elements $\forall a,b\in\mathbb{R}$ are equivalent,
            $a\sim b$, if and only if there exists $k\in\mathbb{Z}$ such that
            $a-b=Tk$.
    \end{itemize}
    Since the transformation~(\ref{eq:alpha}) is a diffeomorphism, there
    exists the inverse transformation
    \begin{equation}
    \label{eq:tran_coords_gen}
    \tau=\pi\left(x\right)\in\left[0,T\right),\quad\xi=\beta\left(x\right)\in\mathbb{R}^{n-1}.
    \end{equation}
    Using this transformation, we define the new control signal $w\in\mathbb{R}$
    as
    \begin{equation}
    \label{eq:fb_trans}
        w:=u-u_{\star}\left(\tau\right)=u-u_{\star}\left(\pi\left(x\right)\right).
    \end{equation}
    Then, the evolution of the coordinates $\tau$ and $\xi$ can be written
    as
    \begin{align}
        \left(\begin{array}{c}
        \dot{\tau}\\
        \dot{\xi}
        \end{array}\right) & =J^{-1}\!\left(\tau,\xi\right)
        \Big(f\left(x\right)+g\left(x\right)u_{\star}\left(\tau\right)+g\left(x\right)w\Big)_{x=\alpha(\tau,\xi)}\nonumber \\
         & =:\left(\begin{array}{c}
        1+f_{\tau}\left(\tau,\xi\right)\\
        f_{\xi}\left(\tau,\xi\right)
        \end{array}\right)+\left(\begin{array}{c}
        g_{\tau}\left(\tau,\xi\right)\\
        g_{\xi}\left(\tau,\xi\right)
        \end{array}\right)w,\label{eq:time-evolution-of-tau-xi}
    \end{align}
    where the Jacobian $J\left(\tau,\xi\right):=\frac{\partial\alpha\left(\tau,\xi\right)}{\partial\left(\tau,\xi\right)}$
    is non-degenerate on $\left(\tau,\xi\right)\in\mathbb{R}\times\mathbf{B}_{\varepsilon}\left(0\right)$,
    since $\alpha$ is a local diffeomorphism. The functions $f_{\tau},f_{\xi},$$g_{\tau},g_{\xi}$
    are continuously differentiable and periodic in $\tau$. The functions
    $f_{\tau},f_{\xi}$ satisfy $f_{\tau}\left(\tau,0\right)=0$, $f_{\xi}\left(\tau,0\right)=0$
    for all $\tau$, due to $x_{\star}\left(\tau\right)=\alpha\left(\tau,0\right)$.
    Changing the free variable in the dynamics~(\ref{eq:time-evolution-of-tau-xi}),
    $\frac{d}{dt}\mapsto\frac{d}{d\tau},$ yields the transverse dynamics
    in the form of the general nonlinear system~\cite{banaszuk1995,Surov-2026}:
    \begin{align}
        \frac{d\xi}{d\tau} & =\frac{f_{\xi}\left(\tau,\xi\right)+g_{\xi}\left(\tau,\xi\right)w}{1+f_{\tau}\left(\tau,\xi\right)+g_{\tau}\left(\tau,\xi\right)w}.\label{eq:transverse-dynamics}
    \end{align}
    As noted in~\cite{Leonov-1998}, a continuous feedback $w=w\left(\tau,\xi\right)$,
    that asymptotically stabilizes the trivial solution of dynamics~(\ref{eq:transverse-dynamics}), induces
    the control law
    \[
        u=u_{\star}\left(\tau\right)+w\left(\tau,\xi\right)
    \]
    for the system~(\ref{eq:sys-dynamics}), which guarantees asymptotic orbital stability of the trajectory $x_{\star}$ in some neighborhood $U_{\rho}\!\left(x_{\star}\right)$. Thus, the problem of orbital
    stabilization of $x_{\star}$ can be reduced to stabilization of the
    trivial solution of system~(\ref{eq:transverse-dynamics}). 
    
    A typical approach~\cite{Finet-2015,Surov-2020} for constructing the feedback $w\left(\tau,\xi\right)$
    is based on the analysis of the linearization of~(\ref{eq:transverse-dynamics}) near $\xi=0$ and $w=0$:
    \begin{align}
    \label{eq:lin-trans-dyn}
        \frac{d\xi}{d\tau}=A\left(\tau\right)\xi+B\left(\tau\right)w,
    \end{align}
    where $A\left(\tau\right) := \partial_{\xi}f_{\xi}\left(\tau,0\right)$, $B\left(\tau\right):=g_{\xi}\left(\tau,0\right)$ are periodic matrices. 
    As shown in~\cite[Theorem~1]{yakubovich1986}, controllability over one period of system~(\ref{eq:lin-trans-dyn})
    guarantees the existence of a feedback
    \[
        w\left(\tau,\xi\right)=K\left(\tau\right)\xi
    \]
    with $T$-periodic gains $K\left(\tau\right)\in\mathbb{R}^{1\times\left(n-1\right)}$.
    The numerical procedure for computing the gains is typically based
    on solving the corresponding LQR problem~\cite{yakubovich1986,gusev2016}. Below, we consider an alternative
    approach, based on event-triggered neuromorphic control.

\section{Trigger-event control for orbital stabilization}
    Influenced by the results in~\cite{heemels2026ejc},
    we demonstrate that the event-triggered neuromorphic control approach
    can also be applied to the problem of orbital stabilization. In particular,
    we consider the problem of stabilizing the transverse dynamics~(\ref{eq:transverse-dynamics})
    using a feedback in impulse form:
    \begin{align}
        u\left(\tau,\xi\right) &= u_{\star}\left(\tau\right)+w\left(\tau,\xi\right),\quad\text{with}\label{eq:impulse-control}\\
        w\left(\tau,\xi\right) &= a\sum_{k}\delta\left(\tau-\tau^{0}_{k}\right)-a\sum_{k}\delta\left(\tau-\tau^{1}_{k}\right),\nonumber 
    \end{align}
    where $\delta$ is the Dirac delta, $a>0$ is a constant representing
    the spike intensity, and $\tau^{0}_{k},\tau^{1}_{k}$ are state-dependent
    event times, which are computed in real time based on the presence
    of certain conditions $\phi_{0}$ and $\phi_{1}$: 
    \begin{align*}
        \tau^{i}_{k} & =\inf_{\tau>\tau_{k-1}}\left\{ \tau\mid\phi_{i}\left(\tau,\xi\right)\leq0\right\} ,\quad i\in\left\{ 0,1\right\} .
    \end{align*}
    Thus, the problem of feedback synthesis reduces to formulating the
    firing conditions $\phi_{i}$. While event-triggered control theory
    is well-developed for LTI systems \cite{petri2025cdc,petri2026nahs,heemels2026ejc,petri2025emulation}, these
    methods cannot be directly extended to the transverse dynamics~(\ref{eq:transverse-dynamics}) for two reasons. First, the linear approximation~(\ref{eq:lin-trans-dyn})
    is obtained by linearizing the closed-loop dynamics along the reference
    orbit, which requires the feedback $w=w\left(\tau,\xi\right)$ to
    be sufficiently smooth. For impulsive control, the closed-loop vector
    field is discontinuous at the event times, and the projection variable
    $\tau$ may jump abruptly,  so expressing the dynamics in terms of $\frac{d}{d\tau}$
    may become problematic. Second, even setting this issue aside and
    assuming the LTV system is valid, event-triggered control algorithms
    for periodic LTV systems are not readily available. These difficulties
    are addressed below.

\subsection{Hybrid transverse dynamics}
    
    To simplify the form of the transverse dynamics of the impulse-controlled
    system, we consider a projection map $\pi$ of a specific form, chosen
    such that the projection variable $\tau$ remains continuous during
    impulses. Specifically, we restrict $\pi$ to depend only on the generalized
    coordinates, $\tau=\pi\left(q\right)$, so that $\tau$ is unaffected
    by the velocity jump induced by an impulse. As shown below, this choice
    yields a well-defined two-phase evolution of the transverse variables.
    \begin{claim}
    \label{claim:hybrid-dynamics} 
        Consider system~(\ref{eq:sys-dynamics}),
        the transverse coordinate transformation~\eqref{eq:tran_coords_gen} given by
        \begin{equation*}
        \tau=\pi\left(q\right),\quad\xi=\beta\left(q,\dot{q}\right),
        \end{equation*}
        and the control law in the form
        \begin{equation}
        \label{eq:sample-control}
            u=u_{\star}\left(\tau\right)+a\delta\left(\tau-\tau_{0}\right),
        \end{equation}
        where $\tau_{0}$ is a fixed constant and $a\ne0$ is spike intensity.
        Then, the evolution of the coordinates $\tau,\xi$ along solutions
        $x\left(t\right)\in U_{\varepsilon}\!\left(x_{\star}\right)$ 
        of~(\ref{eq:sys-dynamics},\ref{eq:sample-control}) is divided into two phases:
        \begin{enumerate}
            \item The \emph{flow phase}, when $\tau\ne\tau_{0}$. The evolution is given
                by the differential equations
                \[
                    \frac{d\tau}{dt}=1+f_{\tau}\left(\tau,\xi\right),\quad\frac{d\xi}{dt}=f_{\xi}\left(\tau,\xi\right).
                \]
            \item The \emph{jump phase}, when $\tau$ reaches $\tau_{0}$. Let this
                phase occur at $t=t_{0}$. Then, the jump map of the projection variable
                is the identity:
                \begin{equation}
                    \tau\left(t^{+}_{0}\right)=\tau\left(t_{0}\right).\label{eq:tau-jump-map}
                \end{equation}
                The jump map of the transverse coordinates is
                \begin{equation}
                \label{eq:xi-jump-map}
                    \xi\left(t^{+}_{0}\right)=\xi\left(t_{0}\right)+a\frac{g_{\xi}\left(\tau_{0},\xi\left(t_{0}\right)\right)}{1+f_{\tau}\left(\tau_{0},\xi\left(t_{0}\right)\right)}.
                \end{equation}
        \end{enumerate}
    \end{claim}
    \begin{proof}
        The flow phase is obtained from~(\ref{eq:transverse-dynamics}) by
        substituting $w=0$. To compute the jump map of $\tau$, we note that
        $\pi$ does not depend on $\dot{q}$, which implies 
        \[
            g_{\tau}\left(\tau,\xi\right)=\frac{\partial\pi}{\partial\dot{q}}M^{-1}\!\left(q\right)F\!\left(q\right)=0.
        \]
        After integrating the resulting expression for $\dot{\tau}$, we obtain~(\ref{eq:tau-jump-map}).
        This result implies that the variable $\tau$ changes continuously
        with time and is strictly increasing (in a small enough neighborhood of the orbit, due to the properties of the
        transverse coordinates). Consequently, the jump map~(\ref{eq:xi-jump-map}) follows
        from direct integration of $\dot{\xi}$ over $\left[t_{0}-\epsilon,t_{0}+\epsilon\right]$
        and the composition rule for the Dirac delta
        \[
            \lim_{\epsilon\to0}\int^{t_{0}+\epsilon}_{t_{0}-\epsilon}f\left(t\right)\delta\left(\tau\left(t\right)-\tau_{0}\right)dt=\frac{f\left(t_{0}\right)}{\dot{\tau}\left(t_{0}\right)}.
            \vspace{-14pt}
        \]
        
    \end{proof}
    Claim~\ref{claim:hybrid-dynamics} allows us to formulate the hybrid dynamics of system~(\ref{eq:sys-dynamics})
    with impulse control~(\ref{eq:impulse-control}). Noting that, for a sufficiently small $U_\rho(x_\star)$ the projection variable is strictly increasing, we can write the transverse dynamics in the form
    \begin{align}
        \frac{d\xi}{d\tau} & =A\left(\tau\right)\xi+r_{1}\left(\tau,\xi\right),\quad\text{when}\quad\tau\notin\left\{ \tau^{0}_{k},\tau^{1}_{k}\right\}, \label{eq:hybrid-transverse-dynamics}\\
        \xi\left(\tau^{+}\right) & =\xi\left(\tau\right)+aB\left(\tau\right)+ar_{2}\left(\tau,\xi\right),\quad\text{when}\quad\tau=\tau^{0}_{k},\nonumber \\
        \xi\left(\tau^{+}\right) & =\xi\left(\tau\right)-aB\left(\tau\right)+ar_{2}\left(\tau,\xi\right),\quad\text{when}\quad\tau=\tau^{1}_{k},\nonumber 
    \end{align}
    where $r_{1}$ collects quadratic and higher-order terms in $\xi$,
    and $r_{2}$ contains linear and higher-order terms in $\xi$. The
    algorithm for stabilizing this hybrid system, i.e., computing the
    values $\tau^{0}_{i},\tau^{1}_{i}$ is based on the  \emph{Lyapunov-based
    greedy approach}~\cite{heemels2026ejc} and is presented below.

\subsection{Lyapunov-based greedy approach for periodic LTV systems}

    To construct the control law, we employ a spike-generation strategy
    based on Lyapunov function greedy approach. This involves constructing
    a quadratic Lyapunov function of the form $V\left(\tau,\xi\right)=\xi^{\top}P\left(\tau\right)\xi$ with positive definite matrix $P(\tau)$.
    The firing conditions $\phi_{i}$ are formulated such that a control
    spike instantaneously decreases the value of $V$. The following theorem
    establishes practical stability of the resulting event-triggered control
    system.
    \begin{thm}
        \label{thm:ltv_stab}
        Consider the hybrid system
        \begin{align}
            \frac{d\xi}{d\tau} & =A\left(\tau\right)\xi,\quad\text{when}\quad\tau\notin\left\{ \tau^{0}_{k},\tau^{1}_{k}\right\}, \label{eq:hybrid-ltv-system}\\
            \xi\left(\tau^{+}\right) & =\xi\left(\tau\right)+aB\left(\tau\right),\quad\text{when}\quad\tau=\tau^{0}_{k},\nonumber \\
            \xi\left(\tau^{+}\right) & =\xi\left(\tau\right)-aB\left(\tau\right),\quad\text{when}\quad\tau=\tau^{1}_{k},\nonumber 
        \end{align}
        where $A\left(\tau\right)$ and $B\left(\tau\right)$ are continuous
        $T$-periodic matrix functions and $a>0$ is the spike intensity.
        Let $P\left(\tau\right)$ be the stabilizing periodic solution\footnote{$P(\tau)$ is a positive definite, symmetric, $T$-periodic matrix function.} to
        the matrix Riccati differential\footnote{For
        brevity, the arguments $\tau$ of the matrix functions $A,B,P$ are omitted.} equation
        \begin{equation}
        \label{eq:ricc}
            \frac{d P}{d\tau}=-A^{\top}P-PA+\frac{1}{r}PBB^{\top}P-Q,
        \end{equation}
        where $Q=Q^{\top}$, $Q>0$ and $r>0$ are design parameters. Then,
        the trivial solution $\xi=0$ of the system with trigger event conditions
        \begin{align}
            \tau^{0}_{k} & =\inf_{\tau>\tau^{0}_{k-1}}\!\left\{ \tau\mid\phi_{0}\left(\tau,\xi\right)\leq0\right\}, \label{eq:trigger-event-feedback}\\
            \tau^{1}_{k} & =\inf_{\tau>\tau^{1}_{k-1}}\!\left\{ \tau\mid\phi_{1}\left(\tau,\xi\right)\leq0\right\} ,\nonumber 
        \end{align}
        where
        \[
            \phi_{i}\left(\tau,\xi\right) := \left(-1\right)^{i}2a\,\xi^{\top}P\!\left(\tau\right)B\!\left(\tau\right)+a^{2}B^{\top}\!\left(\tau\right)P\!\left(\tau\right)B\!\left(\tau\right),
        \]
        is practically stable~\cite{Chaillet-2008}.
    \end{thm}
    \begin{proof}
    \label{proof:stability-theorem} 
        The proof is based on the analysis
        of the evolution of the Lyapunov function $V\left(\tau,\xi\right)=\xi^{\top}P\left(\tau\right)\xi$
        as $\tau$ grows. During the flow phase, both firing conditions satisfy
        $\phi_{0},\phi_{1}>0$, which yields
        \[
            \left|\xi^{\top}PB\right|<\frac{a}{2}B^{\top}PB.
        \]
        The rate of $V$ is therefore 
        \[
            \frac{dV}{d\tau}=\frac{1}{r}\xi^{\top}PBB^{\top}P\xi-\xi^{\top}Q\xi\leq\frac{a^{2}}{4r}\left(B^{\top}PB\right)^{2}-\xi^{\top}Q\xi.
        \]
        Since $P\left(\tau\right)$ is positive definite, continuous, and
        periodic, there exist $p_{\min},p_{\max}>0$, such that 
        \begin{equation}
            p_{\min}\mathbf{I}\leq P\left(\tau\right)\leq p_{\max}\mathbf{I}.\label{eq:P-inequality}
        \end{equation}
        Hence, with $k:=\frac{\lambda_{\min}\left(Q\right)}{p_{\max}}>0$
        we have $Q\geq kP\left(\tau\right)$ for all $\tau$, and consequently
        \begin{equation}
            \frac{dV}{d\tau}\leq\frac{a^{2}}{4r}\left(B^{\top}PB\right)^{2}-kV.\label{eq:lyap-fun-deriv}
        \end{equation}
        The first term on the right-hand side is a continuous periodic function
        of $\tau$ and is therefore bounded. Defining
        \[
            C:=\frac{a^{2}}{4r}\max_{\tau\in\mathbb{R}}\left(B^{\top}PB\right)^{2}\in\left(0,\infty\right),
        \]
        the Gr{\"o}nwall--Bellman inequality gives
        \[
            V\left(\tau\right)\leq V\left(\tau_{0}\right)e^{-k\left(\tau-\tau_{0}\right)}+\frac{C}{k}.
        \]
        
        At a spike, the jump in $V$ is
        \[
            V\left(\tau^{+}\right)=V\left(\tau\right)\pm2a\xi^{\top}PB+a^{2}B^{\top}PB,
        \]
        where the sign is determined by the corresponding firing condition~(\ref{eq:trigger-event-feedback}).
        Rearranging the terms, we have
        \begin{align*}
            V\left(\tau^{+}\right)-V\left(\tau\right) & =\phi_{i}\left(\tau,\xi\right)\leq0.
        \end{align*}
        Thus, $V$ does not increase at spike events. It follows that $V$ converges exponentially to the interval $\left[0,\frac{C}{k}\right]$.
        In view of~(\ref{eq:P-inequality}), for any $r>\frac{C}{kp_{\min}}$,
        the state variable $\xi$ enters the ball $\mathbf{B}_{r}\left(0\right)$
        in finite time and remains there. Hence, the trivial solution is practically
        stable.
    \end{proof}

    \begin{rem}
        Practical stability of the closed-loop linearized system~(\ref{eq:hybrid-ltv-system}) provides a basis for expecting practical stability of the corresponding nonlinear hybrid transverse dynamics~(\ref{eq:hybrid-transverse-dynamics}) in some neighborhood of the orbit.
        A rigorous analysis of the nonlinear hybrid system and its robustness
        to disturbances and unmodeled dynamics is beyond the scope of this
        conference paper and is left for future work.
    \end{rem}
    \begin{rem}
        Practical stability of the trivial solution of the transverse dynamics~(\ref{eq:hybrid-transverse-dynamics})
        implies practical orbital stability of the trajectory $x_{\star}$.
        Indeed, by the definition of practical stability, the transverse coordinates
        $\xi$ converge in finite time to a ball $\mathbf{B}_{r}\left(0\right)$
        and remain within this ball thereafter. Assuming typical continuity requirements for $\alpha$, there exists $\rho = \rho\left(r\right)>0$ such that
        \[
            \xi = \beta(x) \in\mathbf{B}_{r}\left(0\right) \quad \implies \quad
            \mathrm{dist}\left(x,\,\mathrm{orb}\,x_{\star}\right) < \rho.
        \]
        Consequently, the tubular neighborhood $U_{\rho}\!\left(x_{\star}\right)$
        is reached in finite time and remains within this tube thereafter, which
        establishes practical orbital stability of $x_{\star}$.
    \end{rem}

    The theoretical development above establishes the practical orbital stability of the closed-loop system under the proposed controller~(\ref{eq:impulse-control}). We now demonstrate the practical applicability of the approach through experiments on a physical robot.

\section{EXPERIMENTAL RESULTS}
    The proposed control law is evaluated on the Butterfly robot~\cite{Lynch-1998},
    a well-known benchmark for motion control.
    The robot represents a planar mechanical system: a figure-eight-shaped link is mounted to a fixed base through a single revolute joint and rotates in the vertical plane under a DC motor, while a small spherical ball rolls passively along its surface. The link's angular position is measured by an encoder, and the ball's position is estimated in real time from a vision system.

    Assuming ideal rolling, perfect spherical symmetry of the ball, and rotational symmetry of the mechanism, the configuration space is topologically a torus. It is parameterized by two angles: $\vartheta$, the orientation of the link relative to the base, and $\varphi$, the ball's position in a polar frame attached to the link (Fig.~\ref{fig:butterfly-gen-coords}). 
    The DC motor generates the control torque $u$ at the link's revolute joint axis.

    The robot is an underactuated mechanical system with two degrees of freedom and one control input. Its dynamics are described by the Euler–Lagrange equations~\cite{Surov-2026}, which can be represented in the control-affine form~\eqref{eq:sys-dynamics} with state $x\equiv (\vartheta,\varphi,\dot{\vartheta},\dot{\varphi})^{\top}$ and scalar control input $u$.
    Exact expressions for the functions $f$ and $g$ are lengthy and omitted here for brevity; they can be found in~\cite{sumenkov2026}, and the physical parameters are available at~{\small\url{https://gitlab.com/-/snippets/4807338}.}

    Following the algorithm for computing feasible periodic trajectories for the Butterfly robot~\cite{Surov-2026}, we obtained the trajectory illustrated in Fig.~\ref{fig:ref_traj}. Due to rotational symmetry, the coordinates $\vartheta, \varphi$  advance by the angle $\pi$ over one period $T=2.38$ sec, and the system returns to an identical configuration.

    \begin{figure}[!h]
        \centering
        \includegraphics[width=0.98\linewidth]{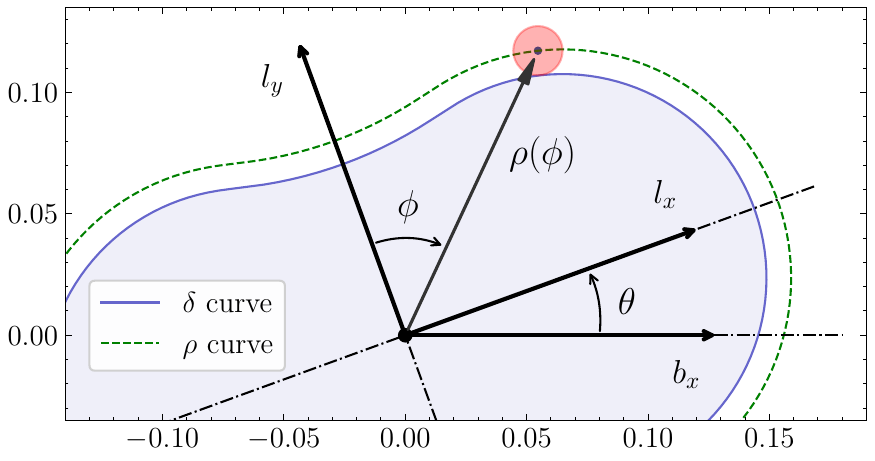}
        \caption{
            Schematic of the Butterfly robot. Generalized coordinates $q \equiv (\vartheta, \varphi)$, where $\vartheta$ is the link orientation and $\varphi$  is the angular position of the ball relative to the link; $\delta$ is the curve defining the shape of the link surface, and $\rho$ is equidistant to it.
        }
        \label{fig:butterfly-gen-coords}
    \end{figure}

    \begin{figure}[!h]
        \centering
        \includegraphics[width=0.99\linewidth]{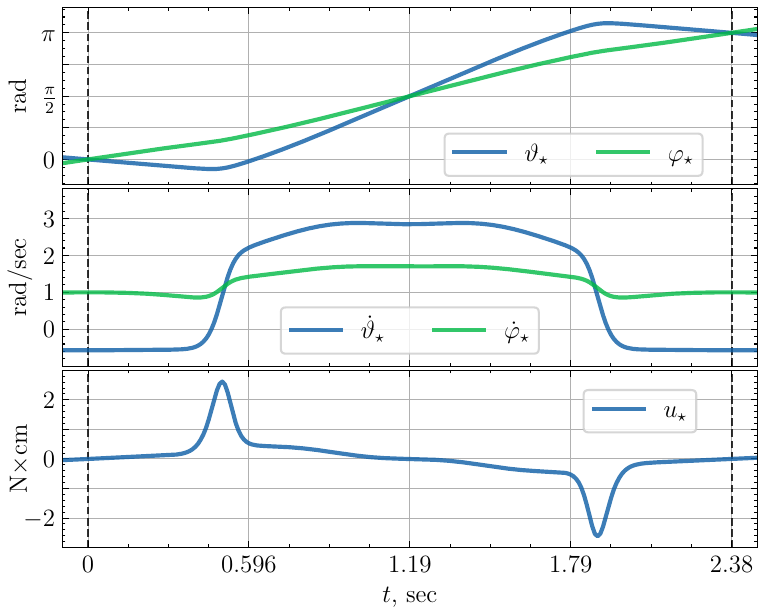}
        \caption{
            Components of the desired solution $x_\star(t)$ with the corresponding nominal control $u_\star(t)$.
        }
        \label{fig:ref_traj}
    \end{figure}
    
\subsection{Periodic LTV system: Transverse linearization}
    
    As seen in Fig.~\ref{fig:ref_traj}, the function $\varphi_\star(t)$ is strictly increasing. Therefore, the projection map $\pi$ can be chosen as
    \begin{equation}
    \label{eq:proj_var}
        \tau = \pi(x) := \varphi_\star^{-1}(\varphi).
    \end{equation}
    The transverse coordinates $\xi = \beta(x) \in \mathbb{R}^3$ are selected as
    \begin{equation}
    \label{eq:tran_coord_butt}
        \xi_1 := \vartheta - \vartheta_\star(\tau), \; \xi_2 := \dot \vartheta - \dot \vartheta_\star(\tau),\; \xi_3 := \dot\varphi - \dot\varphi_\star(\tau),
    \end{equation}
    where $\dot\varphi_\star(\tau) := d\varphi_\star(\tau) / d\tau$, $\dot\vartheta_\star(\tau) := d\vartheta_\star(\tau) / d\tau$. It is straightforward to verify that the inverse transformation $x = \alpha(\tau,\xi)$ exists:
    \begin{align*}
        &\, \vartheta =\xi_{1}+\vartheta_{\star}\left(\tau\right),\quad \varphi=\varphi_{\star}\left(\tau\right), \\
        &\, \dot{\vartheta}=\xi_{2}+\dot{\vartheta}_{\star}\left(\tau\right),\quad \dot{\varphi}=\xi_{3}+\dot{\varphi}_{\star}\left(\tau\right),
    \end{align*}
    and satisfies the conditions given in Section~\ref{sec:transverse-coordinates-def}. Then, following Claim~\ref{claim:hybrid-dynamics}, the robot dynamics can be represented in transverse coordinates in the form~\eqref{eq:hybrid-transverse-dynamics}.

\subsection{Neuromorphic Controller Design}

    The design of the neuromorphic controller for the linearized hybrid transverse dynamics is based on constructing the Lyapunov function $V(\tau, \xi) = \xi^\top P(\tau) \xi$, where $P(\tau)$ is the stabilizing solution to the matrix Riccati differential equation~\eqref{eq:ricc}. We selected the weighting matrix $Q = \mathrm{diag}\{20,\,1,\, 1\}$ and coefficient $R = 2\cdot10^{4}$. Equation~\eqref{eq:ricc} was solved numerically using the LMI approach outlined in~\cite{gusev2016}. The solution $P(\tau)$ along with the functions $B(\tau)$ and $u_\star(\tau)$ was interpolated with a B-spline for real-time computation in the control loop. These functions are available at~{\small\url{https://gitlab.com/-/snippets/6055567}.}

    We conducted a series of experiments on the physical robot, selecting different values of the spike amplitude $a \in \left\{5, 6, 10, 12, 20\right\} \times 10^{-3}$, and for all of these values the algorithm demonstrated good tracking quality. In Figs.~\ref{fig:exp_tran_coords} and~\ref{fig:exp_gen_coords}, we present the experimental results for $a=0.02$. As seen, the transverse variables $\xi$ converge to a neighborhood of the origin, demonstrating practical stability. The corresponding phase trajectory $x(t) \equiv (\vartheta(t), \varphi(t), \dot\vartheta(t), \dot\varphi(t))$ converges to a neighborhood of the reference orbit, thereby demonstrating practical orbital stability.
    In Fig.~\ref{fig:exp_tran_coords}, we also show the control signal $w$ and its filtered value with a cutoff frequency of $20$ Hz. The frequent spike triggering is explained by the presence of disturbances (such as unmodeled contact forces and non-ideal motor dynamics), noise (in the computer vision system), and delays in measurements. These results demonstrate the practical feasibility of the proposed control algorithm and its robustness.

    A video of the experiment with the spiking neuromorphic controller is provided in the multimedia attachment.

    \begin{figure}[!h]
        \centering
        \includegraphics[width=0.99\linewidth]{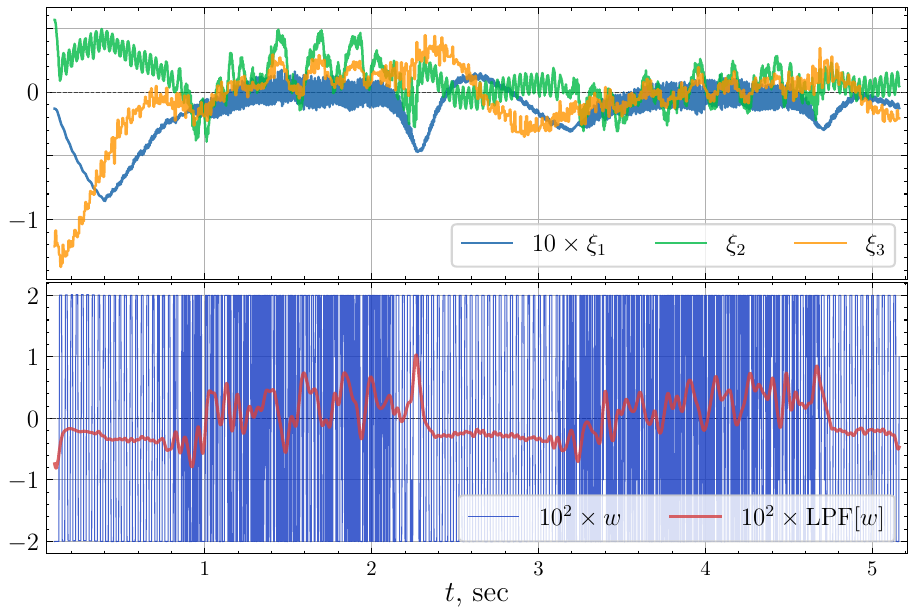}
        \caption{
            Transients of the transverse coordinates $\xi$. Control signal $w$ (blue) and its low-pass-filtered counterpart (red).
        }
        \label{fig:exp_tran_coords}
    \end{figure}

    \begin{figure}[!h]
        \centering
        \includegraphics[width=0.99\linewidth]{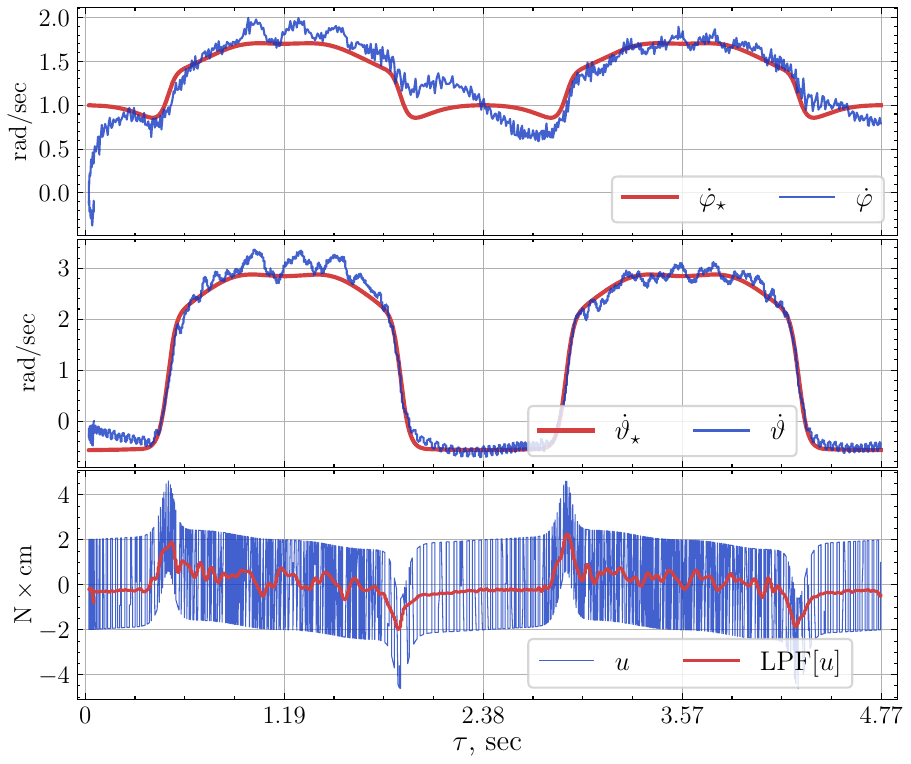}
        \caption{
            Reference signals $\vartheta_\star$, $\varphi_\star$ (red) and measured signals $\vartheta$, $\varphi$ (blue). Control input $u$ (blue) and its low-pass-filtered counterpart (red).
        }
        \label{fig:exp_gen_coords}
    \end{figure}

\section{CONCLUSIONS}

    This paper has provided a theoretical justification that neuromorphic control can solve the problem of periodic orbit tracking. In particular, we have shown that dynamics of an underactuated robot under event-based spiking control can be represented as a linear hybrid system with periodic coefficients. This simplification allowed us to formulate the firing conditions and establish practical orbital stability of the reference trajectory.
    
    We have also demonstrated that the proposed controller design approach is not merely of theoretical interest. Experiments on the physical robot confirm that the resulting controller can be implemented in practice, computed in real time, and remains effective in the presence of disturbances, measurement noise, and delays --- none of which are accounted for in the analysis. Together, the theoretical and experimental results establish both the feasibility and the practical applicability of the proposed neuromorphic control approach.

    The presented results can be readily extended to systems with $m>1$ control signals. A complete theoretical proof of stability and robustness of the nonlinear system, as well as the investigation of different firing strategies, is left for future work.

\addtolength{\textheight}{-16cm}   % This command serves to balance the column lengths
                                  % on the last page of the document manually. It shortens
                                  % the textheight of the last page by a suitable amount.
                                  % This command does not take effect until the next page
                                  % so it should come on the page before the last. Make
                                  % sure that you do not shorten the textheight too much.

%%%%%%%%%%%%%%%%%%%%%%%%%%%%%%%%%%%%%%%%%%%%%%%%%%%%%%%%%%%%%%%%%%%%%%%%%%%%%%%%

\bibliographystyle{IEEEtran}
\bibliography{IEEEabrv,biblio}

\end{document}